\documentclass[]{article}

\usepackage{amsmath}
\usepackage{amssymb}
\usepackage{amsthm}
\usepackage{multirow}

\usepackage{hyperref}
\usepackage{cases}
\usepackage{color}
\usepackage{setspace}
\allowdisplaybreaks[1]

\usepackage{algorithm}
\usepackage{algorithmic}

\usepackage{geometry}

\newcommand{\DS}{\displaystyle}

\newcommand{\reals}{\mathbb R} 
 
\def\bx{\boldsymbol{x}} \def\by{\boldsymbol{y}} 
\def\bX{\boldsymbol{X}}  
\def\bx{\boldsymbol{x}} \def\bv{\boldsymbol{v}} 
\def\be{\boldsymbol{e}}

\def\bV{\boldsymbol{V}} \def\bW{\boldsymbol{W}} 
 \def\bQ{\boldsymbol{Q}} 

\def\bW{\boldsymbol{W}}
\def\bP{\boldsymbol{P}} 
 
\def\bQ{\boldsymbol{Q}} 

\def\bb{\boldsymbol{b}}
 
\def\bSigma{\boldsymbol{\Sigma}}

\def\bomega{\boldsymbol{\omega}}

\def\bLambda{\boldsymbol{\Sigma}} \def\bLambda{\boldsymbol{\Lambda}}
 \def\bI{\boldsymbol{I}} 
\def\cX{\mathcal{X}}  
  \def\cL{\mathcal{L}}
\def\cS{\mathcal{S}}  
  \def\cN{\mathcal{N}}
\def\cY{\mathcal{Y}}  \def\0{\boldsymbol{0}}

\def\tbXi{\widetilde{\bX_{\mathrm{in}}}}

\def\sbxi{\widetilde{\bX_{\mathrm{in}}}}

\def\ni{N_{\mathrm{in}}}
\def\no{N_{\mathrm{out}}}
\def\si{\sigma_{\mathrm{in}}}
\def\so{\sigma_{\mathrm{out}}}

\DeclareMathOperator{\dist}{\mathrm{dist}}
\DeclareMathOperator{\argmin}{\mathrm{argmin}}
\DeclareMathOperator{\argmax}{\mathrm{argmax}}
\DeclareMathOperator{\Tr}{\mathrm{Tr}} \DeclareMathOperator{\Sp}{\mathrm{Sp}}
\DeclareMathOperator{\poly}{\mathrm{poly}}

\DeclareMathOperator{\diag}{\mathrm{diag}}
\DeclareMathOperator{\rank}{\mathrm{rank}}

\DeclareMathOperator{\SNR}{\mathrm{SNR}}
\DeclareMathOperator{\di}{\mathrm{d}}

\newtheorem{theorem}{Theorem}
\newtheorem{example}{Example}
\newtheorem{proposition}{Proposition}
\newtheorem{lemma}{Lemma}
\newtheorem{corollary}{Corollary}

\newtheorem{definition}{Definition}
\newtheorem{assumption}{Assumption}

\title{Robust Subspace Recovery with Adversarial Outliers}
\author{Tyler Maunu \and Gilad Lerman}

\begin{document}

\maketitle

\begin{abstract}
We study the problem of robust subspace recovery (RSR) in the presence of adversarial outliers. That is, we seek a subspace that contains a large portion of a dataset when some fraction of the data points are arbitrarily corrupted. We first examine a theoretical estimator that is intractable to calculate and use it to derive information-theoretic bounds of exact recovery. We then propose two tractable estimators: a variant of RANSAC and a simple relaxation of the theoretical estimator. The two estimators are fast to compute and achieve state-of-the-art theoretical performance in a noiseless RSR setting with adversarial outliers. The former estimator achieves better theoretical guarantees in the noiseless case, while the latter estimator is robust to small noise, and its guarantees significantly improve with non-adversarial models of outliers. We give a complete comparison of guarantees for the adversarial RSR problem, as well as a short discussion on the estimation of affine subspaces.
\end{abstract}



\section{Introduction}

In this paper, we explore some theoretical foundations for the problem of robust subspace recovery (RSR). The mathematical formulation of this problem assumes inliers lying on or around a low-dimensional subspace and outliers lying away from this subspace. The aim is to estimate this subspace and consequently provide a sort of robust variant of the principal component analysis (PCA) subspace. A question present in many past works has been ``what percentage of outliers can my robust subspace estimator tolerate before it fails?''~\cite{xu2012robust,lp_recovery_part1_11,hardt2013algorithms,zhang2014novel,lerman2015robust,zhang2016robust,lerman2017fast,cherapanamjeri2017thresholding,arias2017ransac,maunu2019well}.  This fundamental question is inspired, at least in part, by the study of breakdown points in robust statistics~\cite{donoho1983notion, huber_book,maronna2006robust}, which are defined as the fraction of arbitrary outliers an estimator can tolerate before it fails. We will discuss below what constitutes failure of an estimator. Therefore, the study of breakdown points can be thought of as the study of robustness to \emph{adversarial} outliers because the breakdown point must cover worst-case scenarios. Nevertheless, previous works studying breakdown points in RSR either assumed special non-adversarial settings of outliers or relatively small fractions of adversarial outliers. A more comprehensive study of the problem is missing. This work gives a rigorous and comprehensive study of breakdown points in RSR. 

The above notion of a breakdown point in robust statistics can be more rigorously defined as follows. Suppose that an adversary is allowed to arbitrarily change an $\epsilon$ fraction of the data points. The breakdown point is then the fraction of outliers at which this adversary can make the estimator arbitrarily far from the truth~\cite{Donoho1992,huber_book,maronna2006robust}. In Section  \ref{subsec:background} we provide the typical mathematical definition and explain why it does not apply to RSR.

Nonetheless, for the RSR problem with inliers exactly on the underlying subspace, one alternative notion of a breakdown point is the largest fraction of outliers where perfect estimation of the underlying subspace is still possible~\cite{zhang2014novel}. To be consistent with past literature, we will instead look at lower bounds for the ratio of the number of inliers to the number of outliers. This ratio is referred to as the signal-to noise-ratio (SNR). There is really no distinction between bounding the percentage of outliers from above and bounding the SNR from below because these quantities are in one-to-one correspondence with each other.

Before we go any further, we clarify the RSR problem formulation and introduce some common notation and terminology. The input dataset $\cX$ for the RSR problem is assumed to be an inlier-outlier dataset. That is, it can be partitioned as $\cX = \cX_{\mathrm{in}} \cup \cX_{\mathrm{out}}$, where $\cX_{\mathrm{in}}$ contains inliers lying on or close to a low-dimensional subspace and $\cX_{\mathrm{out}}$ contains outliers distributed in the ambient space. The desired output is the underlying subspace. We denote the number of points in $\cX$ by $N$, the ambient dimension of $\cX$ by $D$ and the dimension of the underlying inlier subspace by $d$. The RSR problem is well-defined if certain conditions are satisfied by the inliers and outliers. Since we wish to address adversarial outliers, we impose no restriction on the outliers, except for an upper bound of their fraction, or equivalently, a lower bound on the SNR. Generally speaking, the inliers need to permeate, or spread throughout, the underlying subspace. In particular, they should not concentrate on its strictly lower-dimensional subspaces.

Different quantitative notions of permeance of inliers are reviewed by~\cite{lerman2018overview}, and the ones relevant to our analysis are described in Section  \ref{sec:identhard}. The RSR formulation with adversarial outliers fits into the $\epsilon$-corruption model specified in previous works~\cite{diakonikolas2018robustly,steinhardt2018resilience} when $| \cX_{\mathrm{out}}| = \epsilon |\cX|$ and the outliers are distributed adversarially. The case where inliers lie exactly on the underlying subspace will be referred to as the noiseless RSR setting, and the case when inliers lie close to this subspace will be referred to as the noisy RSR setting. This work is primarily concerned with computational limits in the noiseless RSR setting. To supplement this, there is a brief discussion on the case of very small noise, whereas the general noisy RSR setting is still generally open and left for future work. As we clarify later, even the problem of noiseless RSR is computationally hard when considering worst-case scenarios. It is thus interesting to guarantee the performance of fast algorithms for sufficiently general, though still special, settings of this problem.

Past works on RSR with adversarial outliers achieve an $O(d)$ lower bound on the SNR for exact recovery~\cite{xu2012robust,cherapanamjeri2017thresholding}. The varying constants in these works depend on how well inliers permeate the underlying subspace, as well as built in deficiencies of the methods pursued. Indeed, the constants are usually quite poor:~ for example, \cite{xu2012robust} achieves an SNR of $121 \mu d/9$ at best in the noiseless case, where $\mu$ is an inlier incoherence parameter. It is unclear how $\mu$ scales in general, and it is known that it can scale like $\log(N)$ in benign cases. The algorithm in \cite{cherapanamjeri2017thresholding} achieves the even worse bound of $128\mu^2 d - 1 $. We will show that it is possible to achieve bounds much closer to $d$ with practical algorithms.

We first quantify the information-theoretic lower bound on the SNR for exact recovery by considering an intractable theoretical estimator. In the case of adversarial outliers and well-distributed inliers this bound is $O(1)$. Higher bounds are provided for more general cases of inliers. Despite the $O(1)$ information-theoretic bound, an $O(d)$ computational-hardness bound is expected for a polynomial-time algorithm when considering a worst-case scenario. After clarifying these issues, we propose two algorithms that obtain the $O(d)$ bound with a relatively small constant. One of them relaxes the intractable theoretical estimator and obtains both state-of-the-art speed and robustness for adversarial outliers and permeating inliers. The other one is a RANSAC-type estimator, which obtains an even lower SNR bound and comparable speed for adversarial outliers and permeating inliers. Its minimal SNR for exact recovery can be $cd$, where $c$ is at least $\Omega(1/\poly\log(d))$. However, unlike the relaxed estimator, the RANSAC-type estimator can be sensitive to small noise, and the RANSAC SNR bound does not improve with non-adversarial outliers.

In the following subsections, we will set-up and outline the contributions of this work.  First, Section  \ref{subsec:background} will give necessary background to understand the contributions of this work. Then, Section  \ref{subsec:contributions} will detail the primary contributions of this work and give an overview of what is covered. Finally, Section  \ref{subsec:notation} will define the notation used in this paper.

\subsection{Necessary Background}
\label{subsec:background}

The RSR problem requires estimation of a low-dimensional subspace with a fixed dimension $d$. For simplicity, we will assume that the underlying $d$-dimensional subspace is linear and refer to such a subspace as a $d$-subspace. Consideration of affine subspaces is left to Section  \ref{sec:centering}. The estimation problem takes place on the set of $d$-subspaces, which is called the Grassmannian manifold. When we write the Grassmannian as $G(D,d)$, we mean the set of $d$-subspaces in $\reals^D$ (where both $d$ and $D$ are fixed).

The theoretical estimator we consider in this paper, and refer to as the $\ell_0$-estimator, is given by a solution to
\begin{equation}\label{eq:l0}
\argmax_{L \in G(D,d)} \left| \cX \cap L \right|.
\end{equation}
This estimator is labelled as an ``$\ell_0$-estimator" because one can equivalently formulate it by
\begin{equation}\label{eq:l0min}
	\argmin_{L \in G(D,d)} \left\|[\dist(\bx_1, L), \dots, \dist(\bx_N, L)]^\top \right\|_0,
\end{equation}
where $\dist(\bx_i, L) = \|\bx_i - \bP_L \bx_i\|_2$ is the Euclidean distance between $\bx_i$ and its orthogonal projection onto $L$, and $\|\cdot\|_0$ denotes the number of nonzero elements in a vector ($\ell_0$-norm).
This estimator, which was introduced in~\cite{lp_recovery_part1_11}, tries to find the $d$-subspace that contains as many points as possible, or equivalently, whose complement contains as few points as possible. This $d$-subspace is also called the \emph{most significant subspace}\footnote{This is in reference to the fact that the dataset could lie within the union of many $d$-subspaces. In this case, the $d$-subspace that contains the most points is more significant than the others.}.
In theory, it is a natural estimator for the underlying subspace in the noiseless RSR setting, and we thus find it useful to examine its information-theoretic limits. In practice, however, it is intractable to calculate. There is a definitive gap between the information-theoretic limits of this estimator and what is achievable by polynomial-time algorithms.

The mathematical definition of a breakdown point is as follows. Suppose we have a set of random points, $y_1, \dots, y_n$, in a metric space $(M,d)$. The breakdown point of an estimator $\hat \alpha(y_1, \dots, y_n)$ for a true parameter $\alpha^*$ is the infimum of the fraction $\epsilon$ at which, for $y_1, \dots, y_{\epsilon n}$ adversarially modified,
$$
    \sup_{y_1, \dots, y_{\epsilon n} \in M} d( \hat \alpha(y_1, \dots, y_n),  \alpha^* ) = \infty.
$$
Extensive works reviewing this and other topics in robust statistics have been written by~\cite{huber_book} and \cite{maronna2006robust}.

Notice that this notion does not immediately extend to estimation on the Grassmannian manifold, since the Grassmannian is compact and common metrics on it are bounded.
We instead work with the following notion of RSR breakdown point. Leaving more technical discussion for later, assume that we have a metric on $G(D,d)$, given by $g: G(D,d) \times G(D,d) \to [0, C]$ (the choice of metric in our work yields $C=\pi/2$). Suppose that we now have an inlier-outlier dataset $\cX$ with $| \cX_{\mathrm{out}} | = \epsilon \cdot |\cX|$, for some $\epsilon$ that is not fixed for this discussion. In this case, we define the RSR breakdown point of an estimator $\hat L$ with respect to the inliers $\cX_{\mathrm{in}}$ on $G(D,d)$ to be the infimum of the SNR (which varies with $\epsilon$) where
\begin{equation}\label{eq:breakdown}
    \sup_{\cX_{\mathrm{out}} \subset \reals^D \setminus L_*} g(\hat L, L_*) > 0.
\end{equation}
This concept of RSR breakdown point was first discussed in \cite{hardt2013algorithms}, though the authors considered the supremum of the fraction of outliers instead of the infimum of the SNR.
This notion appeared before, though not named as a breakdown point, in \cite{xu2012robust} and the earlier preprints of \cite{lp_recovery_part1_11,zhang2014novel,lerman2015robust}. A different, though unconvincing, notion of a breakdown point in RSR appeared in \cite{xu2013outlier}.
Rather than being concerned with a breakdown point at which things are arbitrarily bad, the RSR breakdown point specifies when exact recovery is no longer possible. Exploration of other notions of breakdown points in subspace estimation is an interesting avenue for further research.

In the adversarial setting, where no restrictions are imposed on the outliers, the breakdown point must depend on properties of the inlier dataset, $\cX_{\mathrm{in}}$. As we will see later, different breakdown points are achievable based on what assumptions one makes about $\cX_{\mathrm{in}}$.

\subsection{Contributions of This Work}
\label{subsec:contributions}

The contributions of this paper are summarized in the following points.
\begin{enumerate}
    \item We derive information-theoretic breakdown points for RSR by considering the $\ell_0$-estimator.
    \item We present a RANSAC algorithm that attempts to find the $\ell_0$-estimator. We derive its competitive breakdown point for sufficiently well-distributed inliers. This algorithm has the benefit of less dependence on inlier statistics but may suffer more in noisy settings, and its breakdown points do not seem to improve with simpler models of outliers.
    \item We give a method that achieves a state-of-the-art RSR breakdown point in certain inlier regimes (when they permeate the underlying subspace) while being computationally efficient and flexible in a wide range of examples. This method is called Spherical Geodesic Gradient Descent (SGGD).
    \item We give an overview of all existing theoretical results on the adversarial RSR problem.
\end{enumerate}

The paper is structured in the following way. First, Section  \ref{sec:review} reviews related work on RSR, with a focus on robustness to adversarial outliers. Then, Section  \ref{sec:identhard} discusses information-theoretic and computational hardness bounds in the RSR problem. After this, Section  \ref{sec:ransac} presents the RANSAC algorithm and its associated theoretical guarantees. Following this, the SGGD method is presented in Section  \ref{sec:sggd}. We compare the theoretical guarantees of all RSR algorithms that are robust to adversarial outliers in Section  \ref{sec:comparison}, and the discussion highlights the upsides and downsides to each of these methods. In Section  \ref{sec:centering} we consider the case where the underlying subspace is affine. Finally, Section  \ref{sec:conclusion} concludes this work and discusses future directions.

\subsection{Notation}
\label{subsec:notation}

The Euclidean norm on $\reals^D$ will be denoted by $\|\cdot\|_2$ or sometimes $\|\cdot\|$. For a vector $\bx \in \reals^D$, $\| \bx \|_0$ denotes its number of non-zero entries. The set $O(D,d) \subset \reals^{D
\times d}$ is the set of semi-orthogonal matrices, whose columns are orthonormal. Projection matrices $\bP_L$ and $\bQ_L$ will denote the orthogonal projection onto $L$ and $L^\perp$, respectively. The largest principal angle between subspaces $L_1, L_2 \in G(D,d)$ is denoted by $\theta_1(L_1, L_2)$, where we recall that this angle is the largest singular value of $\bP_{L_1}-\bP_{L_2}$. Whenever we refer to a dataset $\cX$ or its subsets, we are really referring to the multiset $\cX$ and its subsets, since we allow for duplicate elements. The number of elements in a dataset $\cX$ is written as $|\cX|$.
The spherization operator, $\widetilde{\cdot}$, projects a point to the unit sphere as follows:
$\widetilde{\bx} = \bx / \|\bx\|_2$. For a data matrix $\bX \in \reals^{D
\times N}$, whose columns are the data points $\bx_i \in \reals^D$, we denote by $\widetilde{\bX}$ the matrix with the spherization
operator applied to each data point (column-wise).
The smallest angle between vectors $\bv_1, \bv_2 \in \reals^D$ will be denoted by $\angle(\bv_1, \bv_2)$, which takes values in $[0,\pi]$. The angle between a vector $\bv \in \reals^D$ and a subspace $L \in G(D,d)$ is written as $\angle(\bv, L)$, which takes values in $[0, \pi/2]$. We will define a ball on the Grassmannian by
$$
B(L, \gamma) = \{L' : \theta_1(L',L) < \gamma\}.
\label{}
$$
We write $a \gtrsim b$ if $a \geq C b$, for some absolute constant $C$.

\section{Review of Related Work}
\label{sec:review}

In this section, we will review the most important previous work on RSR with a particular emphasis on adversarial robustness. A comprehensive overview of the RSR problem in general is given in~\cite{lerman2018overview}.

One of the most important topics further explored in our work is the concept of stability constraints on models of data in the RSR problem. We borrow the concept of stability from previous work by~\cite{lerman2015robust,maunu2019well,lerman2018overview}. It is used to make a theoretical data model well-defined for an RSR dataset. To do this, it restricts the so-called alignment of outliers (which is, in some sense, how ``low-dimensional'' they are) and requires permeance of the inliers, which was mentioned earlier. A rigorous treatment of these concepts is given in III-A of~\cite{lerman2018overview}. In this paper, we focus on the case where the magnitudes and directions of the outliers are unrestricted. Therefore, restriction of the alignment of outliers translates to restriction of their number, as will be seen later.

Many works exist on the RSR problem, such as robust covariance estimation \cite{maronna1976robust, tyler_dist_free87, locantore1999robust, maronna2005principal, zhang2016robust}, projection pursuit \cite{friedman1974projection, huber_book, Li_85, Ammann1993, maronna2005principal, choulakian06, kwak08, mccoy2011two}, $L_1$-PCA \cite{baccini1996l1, ke2005robust, yu2012efficient, brooks2013pure, markopoulos2014optimal, markopoulos2018outlier}, least absolute deviations \cite{ding2006r, zhang2009median, mccoy2011two, xu2012robust, zhang2014novel, lerman2015robust, clarkson2015input, lerman2017fast, maunu2019well}, outlier filtering methods \cite{Xu1995, xu2013outlier, you2017provable}, exhaustive subspace search \cite{fischler1981random, hardt2013algorithms, arias2017ransac}, and many others. Due to the NP-hardness and non-convexity involved in this problem, it is unclear what optimal strategies are, both in terms of computational time and accuracy.  Further, the model assumed by the problem itself is rather artificial, and in real applications the data is frequently more messy.  More work must be done on understanding the performance of RSR algorithms on real data examples to drive the field forward.

To the best of our knowledge the estimator in~\eqref{eq:l0} is first mentioned in~\cite{lp_recovery_part1_11} as the global $\ell_0$-subspace. Nevertheless,~\cite{lp_recovery_part1_11} focused on $\ell_p$-subspaces for $p>0$ and advocated the case where $p=1$. For simplicity, we clarify these subspaces when $p=1$. For this purpose, we focus on the minimization formulation of the $\ell_0$-estimator in~\eqref{eq:l0min}. The $\ell_1$-subspace formulated in~\cite{lp_recovery_part1_11} may be considered a relaxation of~\eqref{eq:l0min} and is given by
\begin{equation}\label{eq:lad}
\argmin_{L \in G(D,d)} \sum_{\bx_i \in \cX} \| \bQ_L \bx_i \| =\argmin_{L \in G(D,d)}  \sum_{\bx_i \in \cX} \| \bx_i - \bP_L \bx_i \|.
\end{equation}
This subspace is also referred to as the least-absolute-deviations subspace~\cite{lerman2018overview}.
It can be independently motivated by the analogous geometric median in modeling centers of datasets~\cite{lopuhaa1991breakdown}, since minimizers of~\eqref{eq:lad} could be called ``geometric median subspaces''~\cite{zhang2014novel}.
We remark that \cite{lerman2018overview} mentions a different formulation for an $\ell_1$-subspace that relaxes~\eqref{eq:l0} instead of \eqref{eq:l0min}.
The optimization problem in~\eqref{eq:lad} is NP-hard~\cite{clarkson2015input}.
Minimization of various convex relaxations of this energy can be seen in~\cite{mccoy2011two,xu2012robust,zhang2014novel,lerman2015robust}. Direct minimization of this energy in special settings was considered in~\cite{ding2006r,lerman2017fast,maunu2019well}. Out of these algorithms based on least absolute deviations, only Outlier Pursuit (OP), a convex relaxation of this cost, has guarantees for adversarial outliers~\cite{xu2012robust}.

Another algorithm with similar guarantees for adversarial outliers is Thresholding based Outlier Robust PCA (TORP) \cite{cherapanamjeri2017thresholding}. It iterates between fitting a PCA subspace and rejecting points that are either far from the subspace or are highly incoherent (which avoids the problem of low-dimensional inliers). While its guarantees are generally comparable to those of~\cite{xu2012robust}, for Gaussian noise they are significantly stronger. The TORP algorithm is also computationally efficient, unlike OP. However, this algorithm, as well as OP, requires the percentage of outliers as input.

Maunu et al.~\cite{maunu2019well} studied the energy landscape of the robust least absolute deviations function over the non-convex space $G(D,d)$.
However, as we will discuss, this energy function is not without some issues. In particular, it is not robust to large adversarial outliers.
Indeed, their results work for adversarial bounded outliers, but are sensitive to large differences in scale between outliers and inliers.

Other RSR methods compute the top eigenspaces of robust covariance estimators.
Many examples of robust covariance estimators are built on Maronna's original M-estimator of covariance~\cite{maronna1976robust}, which was extended by~\cite{Donoho1992,stahel1981breakdown,tyler_dist_free87,dumbgen1998tyler,locantore1999robust,visuri2000sign}, among many others.
While some robust covariance estimators are not affected by scale differences between inliers and outliers~\cite{tyler_dist_free87,locantore1999robust,visuri2000sign}, we are not aware of any tractable RSR estimator based on a robust covariance that is guaranteed to be robust to adversarial outliers. In this work we use for initialization purposes the spherical PCA~\cite{locantore1999robust,visuri2000sign,maronna2006robust}, which is the simplest subspace estimator that is based on a robust covariance. It computes the top eigenspace of the spherized sample covariance matrix $\widetilde{\bSigma}$, which is also known as the ``spatial sign matrix''~\cite{visuri2000sign}. Specifically, the spherized sample covariance is given by
\begin{equation}
\widetilde{\bSigma} = \frac{1}{N-1} \sum_i \frac{\bx_i \bx_i^\top}{\| \bx_i \|_2^2}.
\end{equation}

Some works have given theoretical guarantees for algorithms under the assumption of general position outliers~\cite{hardt2013algorithms,zhang2016robust,arias2017ransac} (one possible definition of this notion appears in Definition \ref{def:general_pos_out}). In this case, the outliers can have arbitrary magnitude and may lie close to low-dimensional subspaces, but they are not allowed to have linearly dependent structure. While these distributions can be approximately adversarial, these algorithms cannot deal with arbitrary outlier distributions, which may contain low-dimensional structure.

An example of a method with guarantees for general position outliers is RandomizedFind. It searches for linearly dependent $D$-subsets of points, and thus linearly dependent outliers present a problem~\cite{hardt2013algorithms}.
Another similar method is a RANSAC-type algorithm that sub-samples $(d+1)$ points until a linearly dependent subset is found~\cite{arias2017ransac}. This procedure does not work with adversarial outliers for a similar reason. Indeed, suppose that outliers all lie on a 1-dimensional subspace. Then, if one samples a subset with just two of these outliers, one would return a corrupted estimate of the subspace. We will later discuss modifications of the RANSAC algorithm that work for these degenerate cases up to certain limits.

Recent algorithms in the literature on adversarial robustness can also be used to learn subspaces~\cite{diakonikolas2018robustly}. For example, the Robustly Learning a Gaussian (RLG) algorithm could be used to estimate an underlying low-rank covariance matrix up to $O(\epsilon)$ error when there is an $\epsilon$-percentage of outliers. However, this result requires the inlier data to be Gaussian. Another recent paper by~\cite{steinhardt2018resilience} uses resilience to approximate low-rank matrices in the presence of adversarial outliers. This method, which we call Resilience Recovery (RR), outputs a matrix of higher rank that can approximate a low-rank matrix. For both of these methods, exact recovery of a subspace is not possible.

We are most interested in methods that are efficient computationally. Ignoring for a moment iteration complexity, only a few iterative algorithms achieve the per-iteration complexity $O(NDd)$, which is of the same order as standard PCA. These include methods such as FMS~\cite{lerman2017fast}, GGD~\cite{maunu2019well}, TORP~\cite{cherapanamjeri2017thresholding}, and RANSAC~\cite{arias2017ransac}. While the number of iterations for any of these algorithms may be quite large, they each have conditions where they can bound the rate of convergence, although this is a hard fact to quantify in general.

\section{Information-Theoretic and Computational-Hardness Thresholds in RSR}
\label{sec:identhard}

This section will mainly study information-theoretic thresholds on the SNR under different assumptions and models. To be more precise, for adversarial outliers and specific models of noiseless inliers, we give the order of the smallest SNR, which we refer to as the information-theoretic threshold, under which the noiseless RSR formulation is well-defined. The notion of a well-defined formulation for noiseless RSR can be rigorously defined by requiring that $L_*$ is the unique solution of both \eqref{eq:l0} and its following variant,
\begin{equation}\label{eq:l0_variant}
\argmax_{L \in G(D,d)} \left|  \cX_{\mathrm{in}} \cap L \right|.
\end{equation}

Examples of ill-defined RSR were already discussed at length in \S{III.A} of~\cite{lerman2018overview}. We review them here while using our new definition for a well-defined noiseless RSR problem. In view of this definition, the problem is ill-defined if the inliers are contained in a subspace of dimension less than $d$. Indeed, in this case \eqref{eq:l0_variant} does not have a unique solution, since there are infinite number of $d$-subspaces containing the inliers. Even when the inliers span a subspace of dimension $d$ and the solutions to~\eqref{eq:l0_variant} and~\eqref{eq:l0} are unique, the solution to~\eqref{eq:l0_variant} may not coincide with the solution to \eqref{eq:l0}. An example of such a case is demonstrated in Figure 2b of~\cite{lerman2018overview}. As discussed by the authors, examples like this of ill-defined RSR can be avoided by assuming that the inliers permeate the underlying subspace in addition to restricting the fraction of adversarial outliers. In general terms, permeance of inliers means that they are nicely spread throughout the subspace. This is in contrast to the example in the mentioned figure, where the inliers concentrate on a low-dimensional subspace of $L_*$. In Section  \ref{subsec:perminident} we discuss a specific condition under which the inliers nicely permeate the underlying subspace $L_*$ and establish an information-theoretic threshold of order $O(1)$ under this condition. We remark that~\cite{lerman2018overview} also discusses conditions on the outliers that guarantee that the RSR problem is well-defined. As mentioned earlier, in our setting of adversarial outliers, our only restriction regarding the outliers is an upper bound on their fraction or a lower bound on the SNR. For completeness, we also briefly discuss in Section  \ref{subsec:perminident} the non-adversarial case where in addition to the permeated inliers, the outliers nicely permeate the ambient space. In this case, we easily show that the information-theoretic threshold is arbitrarily small as the sample size increases.

In practice the lowest SNR under which existing algorithms can recover the underlying subspace in well-defined settings is $O(d)$. This is indicated by results for specific algorithms in \cite{xu2012robust}, \cite{cherapanamjeri2017thresholding} and this work. Motivated by this observation, \cite{xu2012robust} claims that the information-theoretic threshold is $O(d)$, where their argument is based on a special example.
This peculiar example is reviewed in Section  \ref{subsec:inlierident}, where we show that it can be modified to result in an arbitrarily large information-theoretic threshold for increasing sample size. We thus cannot accept their argument.
In order to address such examples of high information-theoretic thresholds, we provide in this section a general information-theoretic threshold that applies to the different cases of Section  \ref{subsec:perminident} and Section  \ref{subsec:inlierident}.

We believe that the practical $O(d)$ threshold is due to a computational-hardness threshold. This threshold corresponds to the lowest SNR under which a well-defined RSR is computable in polynomial time. Previous results indicate that, for adversarial outliers, this threshold is of order $O(d)$. We review these results and this notion of threshold in Section  \ref{subsec:hardness}.

\subsection{Information-Theoretic Thresholds for Permeating Inliers}
\label{subsec:perminident}

The notion of permeance of inliers was discussed above. Its main purpose is to guarantee a well-defined setting of the RSR problem by restricting the possible distributions of the inliers. 

Xu et al.~\cite{xu2012robust} ensure permeance in an RSR setting with adversarial outliers by bounding the inlier incoherence, which is a rather technical notion inspired by works in compressed sensing. An upper bound on the incoherence implies that the inliers cannot be concentrated around lower-dimensional subspaces~\cite{xu2012robust}. Other works have used lower bounds on certain directional statistics of the inliers~\cite{lerman2015robust, maunu2019well}, although these may be insufficient for the current study, as we later see in Example~\ref{example:car_2}.

In this current work, we use two different conditions that imply some sort of permeance of inliers. The first condition assumes general position of inliers. Below, we clarify this notion and provide an information-theoretic threshold for general position inliers and adversarial outliers. We later show in Section \ref{sec:ransac} that the setting of general position is natural for providing guarantees to a RANSAC-type algorithm. Another condition that implies permeance of inliers requires the notion of a bounded spherical $d$-condition number. Later, in Section \ref{sec:sggd}, we explain this bound and provide performance guarantees for our proposed SGGD algorithm assuming this bound and adversarial outliers. However, we do not have information theoretic limits for this type of permeance.

We first state the definition of inliers in general position. A $d$-subset is a subset with $d$ points.
\begin{definition}[General position inliers]
	A set of inliers, $\cX_{\mathrm{in}} \subset L_* \in G(D,d)$, is said to lie in general position if all $d$-subsets of them are linearly independent (or equivalently, all $d$-subsets span $L_*$).
\end{definition}
Under the assumption of general position inliers, we obtain an $O(1)$ information-theoretic threshold in the following lemma.
\begin{lemma}\label{lemma:geninlier}
    If we assume that the points in $\cX_{\mathrm{in}}$ are in general position within $L_*$ and $\cX_{\mathrm{out}} \subset \reals^D$, then the RSR problem is well-defined when the number of outliers is less than $ (N-d+1)/2 $.
\end{lemma}
\begin{proof}
	In the worst case, all outliers lie on a single 1-dimensional subspace. In this case, we can just take them all to be the same point. Then, due to the general position assumption of the inliers, we see that any $L \in G(D,d) \setminus \{L_*\}$ contains at most $N_{\mathrm{out}} + d-1$ points. Thus, in order for the RSR problem to be well-defined, the number of inliers must be at least
	\begin{equation}\label{eq:geninlierbd}
		N_{\mathrm{in}} > N_{\mathrm{out}} + d-1.
	\end{equation}
	Plugging in $N_{\mathrm{in}} = N - N_{\mathrm{out}}$, we find
	$$ N_{\mathrm{out}} < \frac{N - d + 1}{2}$$
	is the necessary condition for the RSR problem to be well-defined.
\end{proof}
 In this lemma, using the fact that the results imply $N_{\mathrm{in}} >  (N+d-1)/2 $ and $N_{\mathrm{out}} <   (N-d+1)/2 $ the related SNR bound is
\begin{equation}\label{eq:SNR1}
\SNR > \frac{N+d-1}{N-d+1}.
\end{equation}
As $N \to \infty$, we see that this bound goes to 1.

For comparison, we include a case where $\cX_{\mathrm{out}}$ also satisfies a general position condition. This condition is rigorously defined as follows.
\begin{definition}[$d$-subspace general position outliers] \label{def:general_pos_out}
	A set of outliers $\cX_{\mathrm{out}} \subset \reals^D$ is said to lie in $d$-subspace general position with respect to the inliers if
	$$\max_{L \in G(D,d) \setminus \{L_*\}} |\cX \cap L| = d .$$
\end{definition}\
If we assume that the points in $\cX_{\mathrm{in}}$ are in general position within $L_*$ and the points in $\cX_{\mathrm{out}}$ are in $d$-subspace general position with respect to $\cX_{\mathrm{in}}$, then the RSR problem is well-defined as long as $\left|\cX_{\mathrm{in}}\right| > d$. This is because the condition in Definition~\ref{def:general_pos_out} immediately implies that the only subspace that contains more than $d$ points in this case is $L_*$. The information-theoretic bound in this setting is $\SNR \geq (d+1)/\no$, which implies that one can take $\SNR \to 0$ as $\no$ grows.

\subsection{Information-Theoretic Bounds for More General Inliers}
\label{subsec:inlierident}

In this section, we provide information-theoretic bounds for cases with adversarial outliers when inliers are not necessarily in general position. We first motivate our discussion with an example provided to us by an author of \cite{xu2012robust} in order to explain their claim that the information-theoretic SNR threshold is $d$ in general.

\begin{example} \label{example:car_1}
Consider the case where $N_{\mathrm{in}}$ is divisible by $d$ and the inliers are equally partitioned between the points $\be_1$,..., $\be_d \in \reals^D$, where $\be_i$ is the $i$th unit coordinate vector. We note that if the number of outliers is less than $N_{\mathrm{in}}/d$, that is, the SNR is larger than $d$, then the RSR problem is well-defined. However, if the number of outliers is at least $N_{\mathrm{in}}/d$, that is, the SNR is at most $d$, then the problem is generally ill-defined.
Indeed, in this case if all outliers lie, for example, on the point $\be_{d+1}$, then the solution of \eqref{eq:l0} is either not unique (when the SNR is $d$) or it is not $L_*=\Sp(\be_1,...,\be_d)$, which is the solution of \eqref{eq:l0_variant}.
\end{example}

This example can be easily modified to yield an arbitrarily large information-theoretic SNR threshold.
\begin{example} \label{example:car_2}
Consider the case of $N_{\mathrm{in}} - (d-1)$ inliers equal to $\be_1$ and the other $d-1$ inliers equal to $\sqrt{N_{\mathrm{in}} - (d-1)} \be_j$ for $j=2, \dots, d$, where $d>1$. If we take at least 2 outliers equal to $\be_{d+1}$, then the $\ell_0$-estimator must contain $\be_{d+1}$ in its span. 
The lowest SNR under which the problem is well-defined is $N_{\mathrm{in}}/2$. This threshold can be made arbitrarily large by letting $N_{\mathrm{in}}$ grow.
\end{example}

We note that in both examples the inliers have equal variance in all directions. Nevertheless, the second example indicates that even under this nice condition, the SNR threshold may increase with the number of inliers. 

We also note that estimators that respect scale within the dataset could still result in well-defined settings for Example~\ref{example:car_2} with smaller SNR. However, this is not the concern of our study here, since these estimators are usually broken by taking the magnitude of the outliers to be sufficiently large.

Next, we provide a more general result on the information-theoretic threshold for adversarial outliers that applies to both examples above and also to the setting of Lemma~\ref{lemma:geninlier} with general position inliers.
\begin{lemma}\label{lemma:gensnrbd}
    Assume a noiseless RSR setting such that $0< \min_{\bv \in L_* \cap S^{D-1}} \| \bX_{\mathrm{in}}^\top \bv \|_0$ and such that $L_*$ is the unique solution of \eqref{eq:l0_variant}.
    If
    \begin{equation}\label{eq:snrl0bd}
        \SNR > \frac{N_{\mathrm{in}}}{\min_{\bv \in L_* \cap S^{D-1}} \| \bX_{\mathrm{in}}^\top \bv \|_0},
    \end{equation}
    then the RSR problem is well-defined.
\end{lemma}
\begin{proof}
    Denote $c(d,\ni)= \min_{\bv \in L_* \cap S^{D-1}} \| \bX_{\mathrm{in}}^\top \bv \|_0$. This definition ensures that the maximum number of inliers in any $(d-1)$-subspace of $L_*$ is $\ni - c(d,\ni)$. Thus, for any subspace in $G(D,d) \setminus \{L_*\}$, the maximum number of points it contains is $\no + \ni - c(d,\ni)$. This implies that $L_*$ is the unique solution of \eqref{eq:l0}
    as long as $\ni > \no + \ni - c(d,\ni)$, or $c(d, \ni) > \no$. This yields the SNR bound in the lemma, and the proof is concluded.
\end{proof}

We note Lemma \ref{lemma:geninlier} can be obtained as a special case of Lemma \ref{lemma:gensnrbd} since, if inliers are in general position,
$$\min_{\bv \in L_* \cap S^{D-1}} \| \bX_{\mathrm{in}}^\top \bv \|_0 = \ni - (d-1).$$
This implies that the maximum number of points in any $d$-subspace that is not $L_*$ is at most $N_{\mathrm{out}} + (d-1)$ and consequently results in~\eqref{eq:geninlierbd}.
Similarly, one can recover the optimal SNR estimates for Examples \ref{example:car_1} and \ref{example:car_2} from Lemma \ref{lemma:gensnrbd}.

\subsection{Computational Hardness of Robust Subspace Recovery}
\label{subsec:hardness}

An essential result on RSR is the computational-hardness theorem by Hardt and Moitra~\cite{hardt2013algorithms}. It gives the general bound of $\SNR \geq d/(D-d)$ for the noiseless RSR problem, beyond which the problem becomes Small-Set-Expansion (SSE) hard (see Theorem 1.4 of \cite{hardt2013algorithms} for precise details). Hardt and Moitra further explore the SNR threshold in a setting where both inliers and outliers lie in general position and show that the SNR threshold of $d/(D-d)$ can be achieved by polynomial-time algorithms. In their setting, it is interesting to observe that the SNR can approach zero for small fixed $d$ and arbitrarily large $D$. However, this does not coincide with our earlier result that shows, for general position inliers and outliers, the information-theoretic SNR can be arbitrarily small for arbitrarily large sample size, since our previous result applied to fixed $d$ and $D$.

The computational-hardness threshold $d/(D-d)$ is largest when $D-d = 1$, which implies that $d/(D-d) = d$. This scenario seems to be the most relevant to adversarial outliers. Indeed, the most difficult case with adversarial outliers in our study is when they all lie on a line through the origin, so that a subspace containing this line can be the solution of \eqref{eq:l0}. This can be seen in the proof of Lemma \ref{lemma:geninlier}. In this case, the dataset is contained in a $(d+1)$-subspace, which means that effectively $D=d+1$. We further remark that when $d=D-1$, it is not just SSE-hard, but also NP-hard, to recover $L_*$ past the boundary $\SNR = d$ (see \cite{khachiyan1995complexity} and footnote 2 of \cite{hardt2013algorithms}). We believe that this computational-hardness result is related to the current guarantees for RSR with nicely permeated inliers and adversarial outliers, which are $O(d)$ and not $O(1)$ as in the information-theoretic threshold in Lemma \ref{lemma:geninlier}. We remark though that we can obtain in~Section  \ref{sec:ransac} SNR for recovery by RANSAC which is of order $O(d/\poly\log(d))$ and discuss this issue further in Section  \ref{sec:conclusion}.

\section{Smallest SNR Obtained by an Efficient Algorithm}
\label{sec:ransac}

We demonstrate in this section that RANSAC achieves the smallest existing SNR bound for adversarial outliers in the RSR problem. As mentioned, the RANSAC algorithm of \cite{arias2017ransac} does not work for adversarial outliers. Their algorithm works in the following way. They repeatedly subsample subsets of $d+1$ points until they find a set that is linearly dependent. Then, as long as outliers do not have a linearly dependent structure, this algorithm is guaranteed to output a set of $d+1$ inliers that span the underlying subspace. However, if the outliers lie on a line, then the algorithm could just as easily output a subset with two or more outliers because such a subset would be linearly dependent. Adversarial outliers that have linearly dependent structure when combined with the inliers may also present issues for this method.

In order to find algorithms that work for adversarial outliers, we must modify this algorithm. To accomplish this, we revert to the generic formulation of RANSAC based on~\cite{fischler1981random}, to obtain the RANSAC algorithm for RSR given in Algorithm~\ref{alg:ransac}.

This algorithm takes as input a dataset $\cX$, subspace dimension $d$, tolerance $\tau$, consensus number $m$, and maximum number of iterations $n$. The algorithm repeatedly samples subsets of points and fits a $d$-subspace to them. Note that if the subsampled points are linearly dependent, then the algorithm needs to sample more than $d$ points. This sampling can be done quite easily: points are selected uniformly at random without replacement and added to the subset until they span a $d$-subspace.
Then, the algorithm counts how many points have angle less than $\tau$ to the subspace. If this number is greater than $m$, then the algorithm returns this subspace. Otherwise, the algorithm will store the subspace with the highest consensus number so far, and output that subspace after $n$ iterations.

\begin{algorithm}[ht]
	\caption{RANSAC for RSR}
	\label{alg:ransac}
	\begin{algorithmic}[1]
	\STATE{\textbf{Input:} dataset $\cX$, subspace dimension $d$, tolerance $\tau$, consensus number $m$, max iterations $n$}
	\STATE{\textbf{Output:}  $L_* \in G(D,d)$}
	\STATE{$k = 0$, $i = 1$}
	\WHILE{$i \leq n$}
		\STATE {$\cY \gets$ random subset of $\cX$ that spans a $d$-subspace}
		\STATE {$L = \Sp(\cY)$}
		\STATE {$c = \left|\{ \bx \in \cX : \angle(\bx, L) \leq \tau\} \right|$}
		\IF{$c>k$}
			\STATE{$L_* = L$}
			\STATE{$k = c$}
		\ENDIF
		\IF{$k > m$}
			\STATE{\textbf{return }{$L_*$}}
		\ENDIF
		\STATE {$i \gets i + 1$}
	\ENDWHILE
	\STATE{\textbf{return }{$L_*$}}
	\end{algorithmic}
\end{algorithm}

For the case of noiseless RSR with adversarial outliers, we can fix $m = N/2$ and $\tau=0$. The choice of $m$ can be justified by the information-theoretic bound we derived in Lemma~\ref{lemma:geninlier} for general position inliers. As we will prove in the following, RANSAC actually requires a number of inliers much larger than this in order to bound the number of iterations probabilistically. Thus, this choice of $m$ seems sufficient for the study of the adversarial case.


In terms of theoretical guarantees, we follow the example of~\cite{arias2017ransac} and obtain results for small values of $d^2=o(N_{\mathrm{in}})$. This is a limitation of the RANSAC method, and it is needed to obtain the nice recovery guarantees in this section. However, in practice, especially with very high dimensional data, one is usually concerned with finding very low-dimensional subspaces.

This algorithm can work in cases where the RANSAC algorithm of~\cite{arias2017ransac} cannot. To illustrate this, we will prove the following proposition. We define the set $\cL(\cX) = \{L \in G(D,d) : L = \Sp(\bx_{i_1},\dots,\bx_{i_l}), l \geq d\}$. This is the set of subspaces in $G(D,d)$ that are spanned by points in $\cX$, which is exactly the set of candidates for the RANSAC method.
\begin{proposition}[Iteration Complexity of RANSAC]\label{prop:ransac}
	Suppose Algorithm~\ref{alg:ransac} is run with $m = N/2$ and $\tau$ satisfying the bound
	\begin{equation}\label{eq:tau}
	\tau \leq \min_{ L \in \cL(\cX)} \min_{ \bx \in \cX \setminus L} \angle(\bx,L).
	\end{equation}
	Assume also that the inliers are in general position on $L_*$, the outliers are distributed in $\reals^D \setminus L_*$, $d^2=o(N_{\mathrm{in}})$, and the lower bound on the SNR in~\eqref{eq:SNR1} is satisfied. Then, Algorithm~\ref{alg:ransac} returns $L_*$ w.h.p.~in time that is potentially exponential in $d$. The expected number of iterations is approximately $O((N / N_{\mathrm{in}})^{d})$.
	
	If it is further assumed that $\SNR \geq c(d)d$, for any $c(d)=\Omega(1/\poly\log(d))$ and $d$ sufficiently large, then Algorithm~\ref{alg:ransac} outputs $L_*$ in a polynomial number of iterations w.h.p. If $c(d)$ is a constant, then this reduces to $O(1)$ iterations w.h.p.
\end{proposition}
\begin{proof}
   The proof of this proposition is in large part due to the work of~\cite{arias2017ransac}. Suppose that the inliers are in general position on $L_*$ and the outliers are distributed in $\reals^D \setminus L_*$. Suppose also that $\tau$ satisfies~\eqref{eq:tau} (and, in the noiseless case, we can take $\tau=0$). In this case, notice that $L_*$ is the only subspace in $\cL(\cX)$ such that
	\begin{equation}\label{eq:anglecount}
	| \{\bx \in \cX: \angle(\bx,L_*) \leq \tau \}| \geq N/2.
	\end{equation}
	In other words,
	\begin{equation}
	\max_{L \in \cL(\cX) \setminus \{L_*\}} | \{\bx \in \cX: \angle(\bx,L) \leq \tau  \}| < N/2.
	\end{equation}
	This means that, if $m=N/2$, then the algorithm will return either $L_*$ or terminate after $n$ iterations and return another subspace in $\cL(\cX)$.
	
    We can also quantify how long it takes on average for the algorithm to succeed. If $\alpha>1/2$ is the percentage of inliers, then a good approximation for the number of iterations of Algorithm~\ref{alg:ransac} follows a geometric distribution with parameter $\theta = \binom{\alpha N}{d} / \binom{N}{d}$~\cite{arias2017ransac}. This is the case because, if the algorithm succeeds at some iteration, it selects only $d$-points. The expected number of iterations before success is $O(\alpha^{-d})$~\cite{schattschneider2012enhanced,arias2017ransac}. This is of course exponential in $d$ when considering SNRs approaching the lower bound in~\eqref{eq:SNR1}.
	
    To get a probabilistic statement, let the number of iterations before the algorithm succeeds be a random variable. For any given iteration, the probability of success is $O\left( \alpha^{d} \right)$. Thus, the probability in succeeding in at most $n$ iterations is
	\begin{equation}\label{eq:iter1}
	O\left(1 - \left( 1 - \alpha^{d} \right)^n\right).
	\end{equation}
	If $d$ is small (on the order of a constant), the number of iterations is not too large despite being exponential in $d$. The algorithm outputs $L_*$ w.h.p. for large enough $n$ (exponential in $d$).

    Now we further assume that $\SNR \geq cd$, where we leave $c$'s dependence on $d$ as implied. In this case, the percentage of inliers is at least $cd/(1+cd)$, and the expected number of iterations becomes
	\begin{equation}\label{eq:iter3}
	O\left[\left(\frac{cd}{1+cd} \right)^{-d}\right] = 	O\left[\left(1 + \frac{1}{cd} \right)^d\right] \approx O\left[e^{1/c}\right],
	\end{equation}
	for large enough $d$.
	Plugging this same percentage into~\eqref{eq:iter1}, we can get the probability of succeeding in at most $n$ iterations as
	\begin{equation}\label{eq:iter2}
	O\left(1 - \left( 1 - \left(\frac{cd}{1+cd}\right)^{d} \right)^n\right) = O\left( 1 - \left( 1 - \left(1+\frac{1}{cd}\right)^{-d} \right)^n \right) \approx O\left( 1 - \left( 1 - e^{-1/c} \right)^n\right).
	\end{equation}
	Thus, if $c$ is a fixed constant, then RANSAC will recover $L_*$ with high probability after $n$ iterations, and the expected number of iterations is constant.
	
	We further examine if there are other regimes of $c(d)$ that work for~\eqref{eq:iter3}. Using the Taylor approximation for $\log(1+x)$,
	\begin{equation}
	\left(1+\frac{1}{cd}\right)^{d} = \exp\left( d \log\left(1+\frac{1}{cd}\right) \right) \lesssim \exp\left( d \left( \frac{1}{c(d)d} + O\left(\frac{1}{c(d)^2d^2}\right)\right) \right).
	\end{equation}
	Notice, for example, if $c(d)=1/\log(d)$, then this is $O(d)$. However, if $c(d)$ is the reciprocal of any power of $d$, then this becomes exponential in $d$. This observation yields the final claim in the proposition.
\end{proof}

For our adversarial setting, the hardness results discussed in Section~\ref{subsec:hardness} indicate that RSR becomes NP-hard when the fraction of inliers is $(1-\epsilon) \frac{d}{d+1}$, for a constant $\epsilon > 0$. On the other hand, the bound $\SNR > c(d)d$ corresponds to the fraction of inliers being greater than
\begin{equation}
	\frac{c(d)d}{c(d)d + 1} = \left(1 - \left(1-\frac{d+1}{d+c(d)^{-1}}\right)\right) \frac{d}{d+1}.
\end{equation} 
In the notation of the NP-hardness result, this corresponds to $\epsilon = 1-\frac{d+1}{d+c(d)^{-1}}$. Notice that $\epsilon \to 0$ as $d \to \infty$ if $c(d) = \Omega(1/\poly \log(d))$. On the other hand, Proposition~\ref{prop:ransac} does hold for finite $d$, and the $O(\cdot)$ and $\Omega(\cdot)$ notations are not really asymptotic and are rather used to indicate dependence of the terms on $d$. There thus appears to be some gap between our result and the NP-hardness result in the literature. Exactly rectifying this result for RANSAC with the hardness results of~\cite{khachiyan1995complexity} and~\cite{hardt2013algorithms} remains an open question.


\section{A Competitive Algorithm for Multiple Regimes}
\label{sec:sggd}

In this section we discuss a variant of Geodesic Gradient Descent (GGD)~\cite{maunu2019well} that has good robustness properties to adversarial outliers. While it does not match the SNR bounds of the RANSAC method in Section  \ref{sec:ransac} for adversarial outliers, it achieves better SNR bounds when one begins to put restrictions on outliers. For example, the algorithm presented here achieves almost state-of-the-art bounds for the Haystack Model~\cite{lerman2015robust}. The algorithm also has a very efficient $O(NDd)$ per-iteration complexity and converges linearly under some deterministic conditions.

While this algorithm is non-convex and will not recover global minimizers in general, we can still give deterministic results on recovery. The algorithm we use was given by~\cite{maunu2019well}, with the addition of an initial spherizing step.

We consider an energy function $F: O(D,d) \to [0,\infty)$, which is essentially the energy in~\eqref{eq:lad} with spherized data:
\begin{equation}
\label{eq:sphereenergy}
F(\bV) = \sum_{\bx_i \in \mathcal X}  \frac{\|(\bI - \bV\bV^\top)\bx_i\|}{\|\bx_i\|}.
\end{equation}
Here, we cast this as cost over $O(D,d)$ rather than $G(D,d)$ for convenience, although the formulations are equivalent. This is due to the fact that $\bV \bV^T$ is the orthogonal projection matrix onto the subspace $\Sp(\bV)$. We write the function in this way for easier derivative calculation.

This energy is an intuitive relaxation of the energy in~\eqref{eq:l0min} for the following reasons. First, as discussed in~\cite{lerman2018overview}, least absolute deviations is a natural relaxation of the $\ell_0$ deviations. However, unlike $\ell_0$ deviations, least absolute deviations still is sensitive to the scale of the inliers and outliers. So, the simplest way to remove the scale of the data points is to first spherize them. This leads to a scale-free least absolute deviations algorithm. Another way to motivate the  scale free energy~\eqref{eq:sphereenergy} is that it sums the sines of the angles between the points and the span of $\bV$. That is, 
\begin{equation*}
F(\bV) = \sum_{\bx_i \in \mathcal X}  \sin(\angle(\bx_i, \Sp(\bV))).
\end{equation*}

The idea of spherizing in this way was studied for PCA by~\cite{locantore1999robust,visuri2000sign}. Spherical PCA (SPCA) was actually the algorithm of choice in the comparison of robust procedures in~\cite{maronna2005principal}. However, SPCA has no guarantees for exact recovery, and it can still suffer quite a bit with adversarial outliers.


In order to minimize the energy function~\eqref{eq:sphereenergy}, we directly apply GGD~\cite{maunu2019well} to the spherized dataset $\widetilde \cX$, or equivalently, we consider GGD with the energy function in~\eqref{eq:sphereenergy}. For a review of the geodesics on the Grassmannian, see~\cite{Edelman98thegeometry}. Following the derivation in~\S4.1 of~\cite{maunu2019well}, a (sub)derivative of~\eqref{eq:sphereenergy} with respect to $\bV$ is
\begin{align}
	\frac{\partial}{\partial \bV} F(\bV) &= \sum_{\| \bQ_{\bV} \bx_i \| > 0} -
\frac{1}{\|\bQ_{\bV} \bx_i\|} \frac{\bx_i \bx_i^\top \bV}{\|\bx_i\|}.
\end{align}
Following the theory developed in~\cite{Edelman98thegeometry}, the (sub)gradient in the tangent space of $G(D,d)$ is given by
\begin{align}
	\nabla F(\bV) &=\sum_{\| \bQ_{\bV} \bx_i \| > 0} -\frac{\bQ_{\bV} \bx_i}{\|\bQ_{\bV} \bx_i\|} \frac{\bx_i^\top \bV}{\|\bx_i\|}.
\end{align}
Geodesic gradient steps can be taken in the direction of $-\nabla F(\bV)$ along $G(D,d)$ using this formula. This method is called Spherized Geodesic Gradient Descent, and is abbreviated SGGD.

Notice that all of the theory for GGD extends to the current proposed algorithm as well, as long as $\widetilde \cX$ obeys the deterministic conditions in~\cite{maunu2019well}.

In the following, we will explain the theory behind the SGGD algorithm in the presence of adversarial outliers.
First, in Section  \ref{subsec:welltemp}, we will discuss the well-tempered landscape of~\cite{maunu2019well} and how it extends to SGGD with adversarial outliers. Then, Section  \ref{subsec:convsggd} will state results on the convergence rate of SGGD, which includes a guarantee of linear convergence in certain settings.
After this, Section \ref{subsec:sggdinit} specifies conditions under which SPCA initialization guarantees that SGGD starts close enough to the underlying subspace.
Then, Section  \ref{subsec:completesggd} will present the complete theoretical guarantee for SGGD with adversarial outliers, and Section  \ref{subsec:sggdstat} will explain how SGGD performs for some statistical models of data. Finally, Section  \ref{subsec:sggdnoise} will discuss possible results for SGGD when there is small noise.

\subsection{Generic Well-Tempered Condition}
\label{subsec:welltemp}

The deterministic conditions of \cite{maunu2019well} ensure that the landscape of~\eqref{eq:lad} is well-behaved. By well-behaved, we mean that the energy landscape exhibits a nontrivial basin of attraction around $L_*$, where it is possible to converge to $L_*$ using geodesic gradient descent. Since we just propose to use GGD with an initial spherizing step, all of the theory developed for GGD holds if $\widetilde{\bX}$ satisfies these conditions. As we will see in the following, these conditions are essential for understanding how SGGD performs with adversarial outliers.

We first outline how the conditions of \cite{maunu2019well} extend to the spherized setting. Let $\gamma \in [0,\pi/2]$ be a parameter that determines the size of the basin of attraction around $L_*$. The noiseless stability statistic of the dataset $\cX$ with respect to $L_*$ is defined as
\begin{equation}
\cS(\cX, L_*, \gamma) = \cos(\gamma) \lambda_d \left( \sum_{\cX \cap L_*} \frac{ \bx_i \bx_i^\top}{\| \bx_i \|} \right) - \sup_{L \in B(L_*, \gamma)} \| \nabla F(L; \cX \setminus L) \|_2.
\end{equation}
We are concerned with cases when $\cS(\cX, L_*, \gamma) > 0$. For simplicity of notation in this paper, we will also write $\cS > 0$ and leave the parameters as implied. If $\cS > 0$, then the least absolute deviation energy landscape behaves nicely in $B(L_*, \gamma)$. In order to better understand this condition, the following lower bound for $\cS$ was shown in~\cite{maunu2019well}
\begin{equation}
\cS(\cX, L_*, \gamma) \geq \cos(\gamma) \lambda_d \left( \sum_{\cX \cap L_*} \frac{ \bx_i \bx_i^\top}{\| \bx_i \|} \right) - \sqrt{N_{\mathrm{out}}} \| \bX_{\mathrm{out}}\|_2.
\end{equation}
For the adversarial outlier case, we will work with this bound, as it is tight in the worst cases.

For our purposes, we plug in spherized data to obtain a spherized stability statistic, $\tilde \cS$:
\begin{equation}\label{eq:stab}
   	\tilde \cS(\cX, L_*, \gamma) = \cS(\widetilde{\cX}, L_*, \gamma) \geq \cos(\gamma)  \lambda_d \left( \widetilde{\bX} \widetilde{\bX}^\top \right) -  \sqrt{\no} \| \widetilde{\bX_{\mathrm{out}}} \|_2.
\end{equation}
Theorem 1 of \cite{maunu2019well} applied with the spherized stability statistic is formulated as follows.

\begin{theorem}[\cite{maunu2019well}]\label{thm:landscape}
	Suppose that a noiseless inlier-outlier dataset 
	with an underlying subspace $L_*$ satisfies $\tilde\cS(\cX, L_*, \gamma) > 0$, for some $0<\gamma < \pi/2$. Then,
	all points in ${B(L_*,\gamma)}\setminus \{L_*\}$ have a subdifferential along a geodesic strictly less than $-\tilde\cS(\cX, L_*, \gamma)$, that is, it is a direction of decreasing energy. This implies that $L_*$ is the only local minimizer in ${B(L_*,\gamma)}$.
	
\end{theorem}

Our key observation for SGGD in this paper is that, for any set of outliers, we have the bound $\| \widetilde{\bX_{\mathrm{out}}} \|_2 \leq \sqrt{N_{\mathrm{out}}}$. This means that outliers cannot have arbitrary influence on the statistic in~\eqref{eq:stab}. This reduces then to requiring sufficiently many permeated inliers (i.e. large enough $\lambda_d(\sbxi \sbxi^\top)$) to guarantee that $\widetilde \cS > 0$.


In order to better understand the condition $\widetilde \cS > 0$, we lower bound the inlier permeance term, $\lambda_d(\sbxi \sbxi^\top )$. To do this, we will first show that 
\begin{equation}\label{eq:topeigpigeon}
	\lambda_1 \left( \sbxi \sbxi^\top \right) \geq N_{\mathrm{in}} / d
\end{equation} 
by the pigeonhole principle. Suppose that $L_* = \Sp(\be_1, \dots, \be_d)$, where $\be_j$ is the $j$th coordinate unit vector. Then, each point $\bx_i \in \cX_{\mathrm{in}}$ splits its mass among the first $d$ coordinates. The pigeonhole argument proceeds by the following implication:
$$  \sum_{\widetilde{\cX_{\mathrm{in}}}} \sum_{j=1}^d \| \be_j^\top \widetilde{\bx_i}\|_2^2  = N_{\mathrm{in}} \implies \max_{j = 1, \dots, d} \sum_{\widetilde{\cX_{\mathrm{in}}}} \| \be_j^\top \widetilde{\bx_i} \|_2^2 \geq \frac{N_{\mathrm{in}}}{d}. $$
The argument is finished by noting that $\lambda_1 \left( \sbxi \sbxi^\top \right) = \max_{\bv \in \Sp(\be_1, \dots, \be_d)} \sum_{\widetilde{\cX_{\mathrm{in}}}} \| \bv^\top \widetilde{\bx_i}\|_2^2$.

In order to bound $\lambda_d$ from below, we define the spherical $d$-condition number $\kappa$:
\begin{equation}\label{eq:kappabd}
\kappa_d(\tbXi) = \frac{\lambda_1
	\left(\tbXi \tbXi^\top \right)}{\lambda_d \left( \sbxi \sbxi^\top
	\right)}.
\end{equation}
Combining~\eqref{eq:topeigpigeon} and~\eqref{eq:kappabd} yields
\begin{equation}\label{eq:permbound}
	\lambda_d(\sbxi \sbxi^\top ) \geq \frac{\ni }{d \kappa_d(\sbxi)}.
\end{equation}

This bound constitutes a key insight in this work. It allows us to give general SNR bounds for adversarial outliers, along with interpretable statistics.
For example, using~\eqref{eq:stab}, \eqref{eq:permbound} and the fact that $\| \widetilde{\boldsymbol{X}_{\mathrm{out}}}\|_2 \leq \sqrt{N_{\mathrm{out}}}$, a sufficient condition for $\tilde{\cS}(\cX, L_*, \gamma) > 0$ is
\begin{equation}\label{eq:SNRbdgen}
	\SNR > \frac{d \kappa_d(\sbxi)}{\cos(\gamma)}. 
\end{equation}
The conclusion of Theorem~\ref{thm:landscape} can thus be obtained under this simple condition.

The spherical $d$-condition number is a very natural measure of scale-free permeance on a subspace.
Notice that, if the inliers permeate all directions of $L_*$ equally, then $\kappa_d(\sbxi) \approx 1$. Further, a simple application of Theorem 5.39 of~\cite{vershynin2012introduction} implies that $\kappa_d(\sbxi) = O(1)$ w.h.p.~for large $N$ and small $d$ when the inliers follow an isotropic distribution on $L_*$. By definition, an isotropic distribution is invariant to rotations, and so in this case an isotropic distribution on $L_*$ would be invariant to rotations within $L_*$. Any such isotropic distribution becomes uniform on $L_* \cap S^{D-1}$ after spherization. One example of such a distribution is $\cN(\0, \bP_{L_*})$. However, for large values of $d$, we admit that $\kappa_d(\widetilde{\bX_{\mathrm{in}}})$ may be quite large.

\subsection{Convergence Rate of SGGD}
\label{subsec:convsggd}

Here we examine some sufficient conditions under which we can bound the rate of convergence of SGGD to its limit point. Theorems 3 and 4 of~\cite{maunu2019well} can be immediately extended to the SGGD setting.
\begin{theorem}\label{thm:convrate}
    If $\widetilde \cS(\cX, L_*, \gamma) > 0$, SGGD initialized in $B(L_*, \gamma)$ with step size $s/\sqrt{k}$ converges to $L_*$ with rate $O(1/\sqrt{k})$ for sufficiently small $s$. If it is further assumed that
    \begin{equation}\label{eq:lingradcond}
	\widetilde\cS(\cX,L_*,\gamma) >  8 \max_{L \in B(L_*, \gamma) \setminus \{L_*\}} \left| \cX \cap L\right|,
	\end{equation}
	then SGGD with a piecewise constant step size scheme converges linearly to $L_*$.

\end{theorem}

\begin{proof}
This theorem is easily proven by the following examination. The $O(1/\sqrt{k})$ convergence under the condition $\widetilde \cS > 0$ is immediate from Theorem 3 in~\cite{maunu2019well}. The linear convergence guarantee comes after examining the additional condition in Theorem 4 of~\cite{maunu2019well}:
\begin{align} \label{eq:gradconds}
\inf_{L \in B(L_*,\gamma) \setminus \{L_*\} } \frac{1}{4} \left| \frac{\di}{\di t}F(L(t);\widetilde \cX)|_{t=0} \right| &>   2 \sup_{L \in B(L_*, \gamma) \setminus \{ L_*\}} \sum_{ \cX \cap L } \| \widetilde{\bx_i} \|,
\end{align}
where $L(t)$ is a geodesic on $G(D,d)$ from $L$ through $L_*$ parameterized  by arclength. A discussion of the construction of such a geodesic is given in Appendix A of~\cite{maunu2019well}. Notice that we immediately have
\begin{equation}
2 \sup_{L \in B(L_*, \gamma) \setminus \{ L_*\}} \sum_{ \cX \cap L } \| \widetilde{\bx_i} \|=  2 \max_{L \in B(L_*, \gamma) \setminus \{L_*\}} \left| \cX \cap L\right|,
\end{equation}
which yields the right-hand side of~\eqref{eq:lingradcond}. On the other hand, the absolute derivative term in~\eqref{eq:gradconds} can be bounded below by $\widetilde\cS(\cX,L_*,\gamma)$ (using the argument in Theorem 1 of~\cite{maunu2019well}), which yields the left-hand side of~\eqref{eq:lingradcond}.
\end{proof}

The condition in~\eqref{eq:lingradcond} can be further simplified under more special assumptions. For example, if the inliers lie in general position on $L_*$, we can use the sufficent condition for $\widetilde\cS > 0$ seen in~\eqref{eq:SNRbdgen} to derive a further sufficient condition for~\eqref{eq:lingradcond} to hold
\begin{equation}
    \cos(\gamma)\frac{\ni }{d \kappa_d(\sbxi)}   - \no  > 8\cdot (\no + d-1),
\end{equation}
which can be rearranged to write
\begin{align}\label{eq:SNRbdlinconv}
\SNR  &> \left(7 + 8 \cdot \frac{d-1}{N_{\mathrm{out}}}\right) \cdot \frac{d \kappa_d(\sbxi)}{\cos(\gamma)} \\ \nonumber
&= \frac{7d \kappa_d(\sbxi)}{\cos(\gamma)} + o(1).
\end{align}

We summarize these results as follows. Under the condition of~\eqref{eq:SNRbdgen}, we obtain a sublinear convergence rate of $O(1/\sqrt{k})$ for SGGD with step size $s/\sqrt{k}$. If it is further assumed that~\eqref{eq:lingradcond} holds or if inliers are in general position and~\eqref{eq:SNRbdlinconv} holds, then SGGD with a piecewise constant step size linearly converges to $L_*$.

\subsection{Initialization for SGGD}
\label{subsec:sggdinit}

Surprisingly, almost the same condition that ensures a well-tempered landscape also allows for a good initial estimate by SPCA.
It is a restatement of Lemma 9 in~\cite{maunu2019well} with a spherized dataset.
\begin{lemma}\label{lemma:spca}
    Suppose that, for $0 < \gamma < \pi/4$ and a noiseless inlier-outlier dataset with corresponding data matrices $\bX_{\mathrm{in}}$ and $\bX_{\mathrm{out}}$,
	\begin{equation}\label{eq:spca}
        \frac{\sin(\gamma)}{\sqrt{2}}  \lambda_d \left( \widetilde{\bX_{\mathrm{in}}} \widetilde{\bX_{\mathrm{in}}}^\top \right)  -    \| \widetilde{\bX_{\mathrm{out}}} \|_2^2 >0.
	\end{equation}
	Then, SPCA outputs a subspace $L_{\mathrm{SPCA}}$ such that $\theta_1(L_{\mathrm{SPCA}}, L_*) \leq \gamma$.
\end{lemma}
Again, by noticing that $\| \widetilde {\bX_{\mathrm{out}}}\|_2^2 \leq N_{\mathrm{out}}$, this lemma implies that the outliers cannot have an arbitrarily large influence on the condition in~\eqref{eq:spca}. Using the same pigeonhole argument as in Section~\ref{subsec:welltemp}, a sufficient condition for~\eqref{eq:spca} to hold is
\begin{equation}\label{eq:SNRbdgenpca}
	\SNR > \frac{\sqrt{2}}{\sin(\gamma)} d \kappa_d(\sbxi).
\end{equation}

\subsection{Complete Guarantee for SGGD with Adversarial Outliers}
\label{subsec:completesggd}

For the GGD algorithm to be practical, both of the conditions $\tilde \cS(\cX, L_*, \gamma) > 0$ and~\eqref{eq:spca} need to hold. Thus, looking at the sufficent conditions in~\eqref{eq:SNRbdgen} and~\eqref{eq:SNRbdgenpca}, the bounds coincide when $\gamma$ satisfies $\cos(\gamma)=\sin(\gamma)/\sqrt{2}$, which yields $\cos(\gamma)=1/\sqrt{3}$. From here, we will assume this choice of $\gamma$ unless otherwise stated.

We have the following corollary to Lemma~\ref{lemma:spca} using this choice of $\gamma$ and~\eqref{eq:SNRbdgenpca}.
\begin{corollary}\label{cor:spcaadv}
	If
	\begin{equation}\label{eq:spcainit}
	\SNR \geq \sqrt{3} d \cdot \kappa_d(\sbxi) ,
	\end{equation}
	then SPCA outputs a subspace $L_{\mathrm{SPCA}} \in \overline{B(L_*, \arccos(1/\sqrt{3}))}$.
\end{corollary}

Combining Corollary~\ref{cor:spcaadv} with Theorem~\ref{thm:convrate} and the sufficient conditions in~\eqref{eq:SNRbdgen} and~\eqref{eq:SNRbdlinconv}, we obtain the following proposition. This is the primary generic result of this paper for recovery with adversarial outliers.
\begin{proposition}[Recovery by SGGD]\label{prop:sggd}
	If
	$$ \SNR > \sqrt{3} d  \kappa_d(\widetilde{\bX_{\mathrm{in}}}) ,$$
    then SGGD with SPCA initialization and step size $s/\sqrt{k}$ recovers $L_*$ with $O(1/\sqrt{k})$ convergence rate.

    If we further assume that
    $$ \SNR > \left(7 + 8 \cdot \frac{d-1}{N_{\mathrm{out}}}\right) \cdot \sqrt{3} d \kappa_d(\sbxi),$$
   and inliers lie in general position, then SGGD with a shrinking step size converges linearly to $L_*$.
\end{proposition}
The results of this proposition depend only on the SNR and inlier statistics, and thus imply robustness to adversarial outliers.

Comparing these results to those based on OP~\cite{xu2012robust,cherapanamjeri2017thresholding}, we notice that SGGD has much smaller constants.
The best previous bound by~\cite{xu2012robust} is $121 \mu d / 9$, where $\mu$ is the inlier incoherence parameter. The bound for TORP~\cite{cherapanamjeri2017thresholding} is $128 \mu^2 d - 1$.
OP and TORP have a similar dependence on $d$, but they depend on the inlier incoherence parameter rather than the spherical $d$-condition number in the case of SGGD.  We believe that $\kappa_d$ achieves better scaling than the incoherence parameter $\mu$ for certain types of datasets. For example, for isotropic inliers (restricted to the subspace), we can show that $\kappa_d(\widetilde{\bX_{\mathrm{in}}}) = O(1)$, while $\mu = O(\max(1,\log(N)/d))$, which is much worse for large $N$.

\subsection{Results for Statistical Models of Data}
\label{subsec:sggdstat}

The extensive discussion by~\cite{maunu2019well} on statistical models that satisfy their condition can also be extended to the case of the spherized conditions as well. It is important to show that the algorithm is flexible beyond the case of adversarial outliers where the SNR is typically required to be very high. SGGD can also exactly recover $L_*$ for very low SNRs when one adds restrictions to the outliers (unlike OP and TORP).

We focus here on the Haystack Model to balance out the discussion of adversarial models of data. In this model, inliers are i.i.d.~$\cN(\0,\si^2 \bP_{L_*}/d)$ and outliers are i.i.d.~$\cN(\0,\so^2 \bI/D)$. Notice that, after spherizing, the inlier distribution is uniform on $L_* \cap S^{D-1}$ and the outlier distribution is uniform on $S^{D-1}$. Analogously to GGD, we can consider different regimes of sample size for SGGD in the Haystack model~\cite{maunu2019well}.
To prove these results, one must operate using
\begin{equation}
    \cS(\widetilde \cX, L_*, \gamma) \geq \sqrt{3}  \lambda_d \left( \widetilde{\bX} \widetilde{\bX}^\top \right) - \max_{L \in B(L_*, \gamma) \setminus \{ L \in G(D,d)} \left\| \widetilde{\bQ_L \bX_{\mathrm{out}}} \right\|_2 \left\| \widetilde{\bX_{\mathrm{out}}} \right\|_2,
\end{equation}
rather than the lower bound on $\cS$ seen in~\eqref{eq:spcainit}.
In this case, the following theorem is a direct corollary of the results in~\cite{maunu2019well}.
We note that the big-O notation used here is slightly abused. Namely, it is used to denote the order of the threshold at which exact recovery is possible with high probability in the various sample regimes and not asymptotic limits.
\begin{theorem}
	In the sample regime of $N = O(D)$, SGGD recovers $L_*$ w.o.p.~if
	 \begin{equation}
	 	\label{eq:full_small_regime}
	 	\SNR \geq \max \left(8 \sqrt{2} \frac{d}{\sqrt{D}},  2 \frac{ d }{ D } \right).
	 \end{equation}
	 In the sample regime of $N = O(d(D-d)^2 \log (D))$,  SGGD recovers $L_*$ w.h.p.~if
	 $$\SNR \geq \max \left(   \frac{5\sqrt{2} d}{\sqrt{D(D-d)}} ,   \frac{ 2 d }{D} \right).$$
	Finally, in the sample regime $N_{\mathrm{out}}=O(\max(d^3 D^3 \log^3(N),(dN_{\mathrm{out}}/N_{\mathrm{in}})^6))$, SGGD recovers $L_*$ w.o.p.~for any fixed $\alpha$ and $\SNR \geq \alpha > 0 $.
\end{theorem}
We note that spherizing here has the added benefit of removing the dependence on $\si/\so$. However, this can make the bounds worse if $\si \gg \so$.

\subsection{Extension of SGGD to Noisy Settings}
\label{subsec:sggdnoise}

We finish by noting that the results for GGD in noisy settings can be directly extended to SGGD.
To deal with the noisy case, we resort to a noisy stability statistic. We avoid technical details and instead refer the reader to \S3.1.2  of~\cite{maunu2019well}. Here, the authors define a noisy surrogate to the stability statistic, which is denoted by $\cS_n(\cX,L_*, \epsilon, \delta, \gamma)$. The parameter $\epsilon$ is a uniform bound on the norm of the noise added to the inliers, and $\delta$ is an additional technical parameter. All that is required for the case of SGGD is for the conditions of the theorem to hold for $\widetilde{\cX}$ rather than $\cX$. For example, the following theorem is a restatement of Theorem 2 of~\cite{maunu2019well} with the appropriate modifications for the spherized setting. This theorem guarantees a well-tempered landscape for noisy datasets in the spherized setting.
\begin{theorem}[Stability with Small Noise]\label{thm:landscapenoise}
    Assume a noisy inlier-outlier dataset, with an underlying subspace $L_*$ and noise parameter $\epsilon > 0$, that satisfies for some $\delta > \epsilon > 0$ the stability condition $\cS_n(\widetilde{\cX},L_*, \epsilon, \delta,\gamma) > 0$. Let $\eta = 2 \arctan(\epsilon / \delta)$ and assume  further that $\eta < \gamma$. Then, all points in ${B(L_*,\gamma) \setminus B(L_*,\eta)}$ have a subdifferential along a geodesic strictly less than $-\cS_n(\widetilde{\cX},L_*, \epsilon, \delta,\gamma)$, that is, it is a direction of decreasing energy. This implies that the only local minimizers and saddle points in ${B(L_*,\gamma)}$ are in $B(L_*,\eta)$.
\end{theorem}

As mentioned in~\cite{maunu2019well}, one can prove that GGD approximately recovers the underlying subspace in some noisy RSR setting, although the authors do not explicitly show this result. In the same way, one can prove an approximate recovery result for SGGD in the noisy RSR setting, as long as the noisy stability statistic $\widetilde{S}_n$ is sufficiently large.  
However, it does not seem easy to write simple SNR bounds for the noisy case like we have in Proposition~\ref{prop:sggd}.
Also, we comment that the conditions of Theorem~\ref{thm:landscapenoise} can be shown to hold for small noise in the spherized Haystack Model.
We leave an in depth discussion of the noisy RSR setting to future work.

\section{Comparison of All Theoretical Results}
\label{sec:comparison}

In this section we compare all existing theoretical results for RSR with adversarial outliers. This is done in Table~\ref{tab:SNR}. The algorithms we compare with are SPCA~\cite{locantore1999robust}, SGGD (this work), OP~\cite{xu2012robust}, TORP~\cite{cherapanamjeri2017thresholding}, RANSAC~\cite{fischler1981random} (and this work), Robustly Learning a Gaussian (RLG)~\cite{diakonikolas2018robustly} and Resilience Recovery (RR)~\cite{steinhardt2018resilience}. For each algorithm, we give the associated SNR bound for adversarial outliers, as well as a short description of the upsides and downsides of each algorithm on the inliers. In view of these results, it seems that RANSAC has the best guarantees for adversarial outliers. On the other hand, SGGD is quite competitive, and it actually has theoretical results for small noise that RANSAC does not. Empirical tests also indicate that SGGD more gracefully handles noise than RANSAC. SGGD handles extremely low SNR regimes if the outliers are not aligned, as is the case for the Haystack Model. RANSAC does not have these same guarantees. Finally, SPCA, RLG, and RR all are able to approximate the underlying subspace with adversarial outliers, but they are not able to exactly recover it.

\begin{table}[!ht]
	\footnotesize
	\centering
	\begin{tabular}{|c|l|}\hline
		
		\multirow{3}{*}{\textbf{SPCA}~\cite{locantore1999robust}} & \multicolumn{1}{l|}{$\DS \SNR > \frac{d  \kappa_d(\widetilde{\bX_{\mathrm{in}}})}{ \sin(\gamma) / \sqrt{2}} $} \\ \cline{2-2}
		\multirow{2}{*}{} & \emph{\underline{Upsides:} Fast $\gamma$-approximation.} \\
		& \emph{\underline{Downsides:} No exact recovery (only $\gamma$-approximation).}\\
		\hline \hline
		
		\multirow{4}{*}{\textbf{OP}~\cite{xu2012robust}} & \multicolumn{1}{l|}{$\DS\SNR \geq \frac{121 \mu d}{9}$} \\ \cline{2-2}
		\multirow{3}{*}{} & \emph{\underline{Upsides:} Convex algorithm and nalysis for small general noise.} \\
		& \emph{\underline{Downsides:} Incoherence $\mu$ can be large, poor constants, no} \\
		& \emph{strong noise analysis, requires parameter tuning.} \\
		\hline \hline
		
		\multirow{5}{*}{\textbf{TORP}~\cite{cherapanamjeri2017thresholding}}  &  \multicolumn{1}{l|}{$\SNR \geq 128 \mu^2 d - 1$}   \\ \cline{2-2}
		\multirow{2}{*}{} & \emph{\underline{Upsides:} Nice analysis for Gaussian noise. Analysis for small} \\
                          &\emph{general noise.}\\
		& \emph{\underline{Downsides:} Parameter $\mu$ can be large, poor constants,} \\
				&\emph{requires parameter tuning.} \\
		\hline \hline
		
		\multirow{4}{*}{\textbf{RR} (\cite{steinhardt2018resilience})} & \multicolumn{1}{l|}{$\DS \SNR \geq 2$} \\ \cline{2-2}
		\multirow{2}{*}{} & \emph{\underline{Upsides:} Fast approximation for constant SNR bound,} \\
		&\emph{\underline{Downsides:} Constant approximation, no exact recovery,} \\
		& \emph{returns a matrix of rank at most 15$d$, requires resilient inliers.} \\
		\hline \hline
		
		\multirow{4}{*}{\textbf{RLG} (\cite{diakonikolas2018robustly})} & \multicolumn{1}{l|}{$\DS \SNR \geq \frac{1-\epsilon}{\epsilon}$} \\ \cline{2-2}
		\multirow{3}{*}{} & \emph{\underline{Upsides:} Fast $\epsilon$-approximation for a different problem.} \\
		&\emph{\underline{Downsides:} For RSR, reduces to only Gaussian inliers and} \\
		& \emph{no exact recovery ($\epsilon$-approximation).} \\
		\hline \hline
		
		\multirow{5}{*}{\textbf{SGGD} (\cite{maunu2019well} and this work)} & \multicolumn{1}{l|}{$\DS \SNR > \sqrt{3} d  \kappa_d(\widetilde{\bX_{\mathrm{in}}}) $} \\ \cline{2-2}
		\multirow{4}{*}{} & \emph{\underline{Upsides:} Efficient, linear convergence with another condition,} \\
		& \emph{good constants, adapts to other statistical models of data, analysis} \\
        & \emph{for small general noise.}\\
		&\emph{\underline{Downsides:} No strong noise analysis.} \\
		\hline \hline
		
		\multirow{5}{*}{\textbf{RANSAC} (\cite{fischler1981random} and this work)} &  \multicolumn{1}{l|}{$ \SNR \geq cd$}   \\ \cline{2-2}
		\multirow{2}{*}{} & \emph{\underline{Upsides:} Good constants.} \\
		&\emph{\underline{Downsides:} Potentially sensitive and unstable to noise,} \\
        &  \emph{requires parameter tuning, guarantees require general position} \\
        &\emph{inliers and no noise analysis.}\\
		\hline
		
		
	\end{tabular}
	\vspace{3mm}
	\caption{Table comparing adversarial SNR bounds for RSR algorithms. There are different types of requirements on inliers in this table: $\mu$ is the incoherence parameter~\cite{xu2012robust}, $\kappa_d$ is the spherical $d$-condition number, RLG~\cite{diakonikolas2018robustly} requires Gaussian inliers, RANSAC requires inliers in general position.
		\label{tab:SNR}}
\end{table}

\section{Generalization to Affine Subspaces}
\label{sec:centering}

In this section, we will consider the case of estimating an affine $d$-subspace, that is, a $d$-dimensional affine subspace, in a dataset with outliers. Affine subspaces are represented as equivalence classes $[\bb + L]$ for affine offsets $\bb \in \reals^D$ and linear subspaces $L \in G(D,d)$, where
\begin{equation}
(\bb_1 + L) \sim (\bb_2 + L) \iff \bb_1 - \bb_2  \in L.
\end{equation}

In the following, we will briefly discuss how to extend RANSAC and SGGD to the affine case. First, Section  \ref{subsec:affransac} will discuss an affine variant of RANSAC, and then Section  \ref{subsec:affsggd} will discuss an affine variant of SGGD.

\subsection{An Affine Variant of RANSAC}
\label{subsec:affransac}

We note that the RANSAC method naturally extends to the case of affine subspaces. We can replace step 3 of Algorithm~\ref{alg:ransac} with the following procedure:
\begin{itemize}
	\item Sample a point $\by_0$ uniformly at random from the dataset $\cX$;
    \item Sample $\cY$ as a random $d$-subset of $\cX$;
	\item If $\cY - \by_0$ does not span a linear $d$-subspace, sample more points until it does span a linear $d$-subspace.
\end{itemize}
From the set $\cY$, we can fit an affine $d$-subspace using $\by_0$ as the offset. 
The theoretical results for RANSAC with adversarial outliers immediately extend to the affine case.

\subsection{An Affine Variant of SGGD}
\label{subsec:affsggd}

On the other hand, it is not as obvious how to extend SGGD to the affine case.
In practice, centering by the geometric median or mean may be sufficient for many tasks. However, in terms of theoretical guarantees, neither of these centers would be able to recover an affine subspace in general, especially with adversarial outliers. Further, it is not clear that this centering combined with spherization would yield anything useful at all, since the spherization step could distort the data even more if it is not well centered.

To overcome this, we can use the technique of symmetrization~\cite{dumbgen1998tyler}. A symmetrized dataset $\cX'$ is formed as
\begin{equation}
\cX' \in \reals^{D \times \binom{n}{2}}= \{ \bx_i - \bx_j: \bx_i, \bx_j \in \cX, i < j \}.
\end{equation}
One could then run SGGD on $\widetilde{\cX'}$ and estimate the linear subspace component. Notice that the follow-up computation of the offset for the affine subspace will not affect this portion of the computation and thus does not affect the associated SNR bounds. After estimating this linear subspace, $\hat L$, it is nontrivial to robustly estimate the offset factor in general. Nevertheless, in the noiseless or small noise regimes, one could project the dataset $\cX$ onto $\hat L^{\perp}$, and then the inliers would all cluster at a point. This point could be found by, for example, calculating the geometric median of the projected points. As long as the SNR of the symmetrized dataset is greater than 1 and $\hat{L} = L_*$, then in the noiseless case this geometric median is guaranteed to exactly estimate a point on the affine subspace $\bb^*$. In fact, this $\bb^*$ will be the offset with minimal norm.

One note is that the SNR bound changes for the case of a symmetrized inlier-outlier dataset. Indeed, if we have a dataset of $N_{\mathrm{in}} = \alpha N$ inliers, then the symmetrized dataset would have
\begin{equation}
	\frac{\binom{\alpha N}{2}}{\binom{N}{2}} = \frac{\alpha N (\alpha N - 1) }{N(N-1)} = \frac{\alpha^2 N - \alpha}{N-1} \approx \alpha^2.
\end{equation}
Therefore, the fraction of pure inlier points decreases by a factor of $\alpha$. This then translates into the SNR in the symmetrized case being approximately $\alpha^2/(1-\alpha^2)$. As a consequence, the exact recovery condition in Proposition~\ref{prop:sggd} becomes
\begin{equation}
	\SNR \geq  3 \sqrt{3} d  \cdot \kappa_d(\widetilde{\bX_{\mathrm{in}}'}),
\end{equation}
where $\bX_{\mathrm{in}}' \subset \cX'$ is the set of symmetrized inlier points.

Using the symmetrized data increases the computational complexity of SGGD from $O(NDd)$ to $O(N^2Dd)$.

\section{Conclusion}
\label{sec:conclusion}

In this paper, we have presented some simple results on adversarial outliers in the problem of robust subspace recovery. We study the most significant subspace estimator and its associated information-theoretic SNR lower bounds in this setting. We also give two practical algorithms that achieve state-of-the-art bounds for the adversarial RSR setting. These are a RANSAC based method and the SGGD method. We give a comparison of existing RSR algorithms with guarantees for this setting and find that the SGGD algorithm has the best combination of guarantees.

As we mentioned in Section \ref{sec:ransac}, the smaller SNR threshold of $\SNR > c(d)d$ for RANSAC, where $c(d)=\Omega(1/\poly \log(d))$, seems at odds with the hardness threshold of $\SNR > d$~\cite{khachiyan1995complexity,hardt2013algorithms}.
Exact rectification of these different thresholds remains an open question.

Despite the nice properties of our algorithms, there are many questions that remain unanswered. Most importantly, analysis to noise is lacking, since we only show that the SGGD method is stable to small noise. As we mentioned earlier, it would be interesting to study RSR when there is heavy-tailed noise~\cite{minsker2015geometric} or RSR under the spiked model~\cite{johnstone2001distribution}.

Another unanswered question is the selection of subspace dimension in subspace recovery problems. For PCA, we are aware of a few methods, such as~\cite{kritchman2008determining} and~\cite{dobriban2017factor}. However, there is no obvious way to extend these techniques to the RSR problem.

We also believe that it is important to test RSR algorithms on real data, especially those with a real metric of interest. An example of this is the dimension reduction for clustering experiment in~\cite{lerman2017fast}, where the authors show that initial dimension reduction by RSR can improve $k$-means clustering accuracy.

\bibliographystyle{siamplain}
\bibliography{refs_2017}

\end{document}